\documentclass[11pt]{article}
\usepackage{amsmath, amssymb, amsthm, color, epsfig, epstopdf, enumerate,framed}
\usepackage{graphicx, hyperref, listings, mathtools, multicol, multirow, pdfpages, subfigure}
\usepackage{booktabs, siunitx} 
\usepackage{textgreek}

\usepackage{amsopn, authblk} 
\usepackage[table]{xcolor}
\usepackage[numbers,sort&compress]{natbib} 
\usepackage{appendix}

\usepackage{algorithm, algpseudocode} 

\theoremstyle{plain}

\newcommand{\A}{\boldsymbol{A}}

\newcommand{\bx}{\boldsymbol{x}}

\newcommand{\bomega}{\boldsymbol{\omega}}

\newcommand{\R}{\mathbb{R}}

\newcommand{\bI}{\boldsymbol{I}}

\newcommand{\bA}{\boldsymbol{A}}
\newcommand{\bX}{\boldsymbol{X}}

\newcommand{\vertiii}[1]{{\left\vert\kern-0.25ex\left\vert\kern-0.25ex\left\vert #1 
    \right\vert\kern-0.25ex\right\vert\kern-0.25ex\right\vert}}

\theoremstyle{plain}
\newtheorem{theorem}{Theorem}[section]
\newtheorem*{theorem*}{Theorem}
\newtheorem{lemma}[theorem]{Lemma}

\newtheorem{corollary}[theorem]{Corollary}
\theoremstyle{definition}

\theoremstyle{remark}

\newtheorem{remark}[theorem]{Remark}

\numberwithin{equation}{section}
\numberwithin{algorithm}{section}
\numberwithin{figure}{section}
\numberwithin{table}{section}

\title{Concentration of Random Feature Matrices in High-Dimensions}
\author[1]{Zhijun Chen}
\author[1]{Hayden Schaeffer}
\author[2]{Rachel Ward}
\affil[1]{Carnegie Mellon University}
\affil[2]{The University of Texas at  Austin}
\date{~}

\begin{document}
\maketitle

\begin{abstract}
The spectra of random feature matrices provide essential information on the conditioning of the linear system used in random feature regression problems and are thus connected to the consistency and generalization of random feature models. Random feature matrices are asymmetric rectangular nonlinear matrices depending on two input variables, the data and the weights, which can make their characterization challenging. We consider two settings for the two input variables, either both are random variables or one is a random variable and the other is well-separated, i.e. there is a minimum distance between points.  With conditions on the dimension, the complexity ratio, and the sampling variance, we show that  the singular values of these matrices concentrate near their full expectation and near one with high-probability. In particular, since the dimension depends only on the logarithm of the number of random weights or the number of data points, our complexity bounds can be achieved even in moderate dimensions for many practical setting. The theoretical results are verified with numerical experiments.

 \end{abstract}



\section{Introduction}\label{sec: intro}

Kernel methods are some of the most popular approaches in machine learning and have been applied to image processing, classification, and data-based regression problems. Suppose $K(x,x')$ is the kernel, which measures the similarity between two input samples $x$ and $x'$, then the kernel approach generates an approximation by using a weighted sum of the kernel applied to the data.  This is justified since minimizers of kernel training problems within the reproducing kernel Hilbert space associated with the kernel $K$ are guaranteed to be of the form $\sum_{j=1}^m c_j K(x,x_j)$ by the representer theorem \cite{murphy2012machine}, where $\{x_j\}_{j=1}^m$ are data samples. Standard kernel methods require the formation of the kernel matrix which uses all of the data points and thus can be intractable for large problems. However, if the kernel is symmetric and positive definite, i.e. $K: \mathbb{R}^{2d} \rightarrow \mathbb{R}$, $K(x,x')=K(x',x)$, and $\langle \mathbf{c}, \mathbf{K} \mathbf{c}\rangle>0$ where $\mathbf{K}$ is the matrix whose elements are $K(x_i,x_j)$ where $x_i$ and $x_j$ are data points, then $K(x,x') =\langle \psi(x), \psi(x') \rangle$ with feature map $\psi$. Rather than forming and computing solutions using the kernel matrix directly, the random feature method (RFM) \cite{rahimi2007random, rahimi2008uniform, rahimi2008weighted} uses a randomized approximation to the inner product. That is, the kernel can be approximated by $\Phi(x)^T \Phi(x')$, where $\Phi:\R^d \rightarrow \R^N$ and $N$ is the number of randomized features used. When $N$ is not large, the randomized method leads to a significant reduction in the computational cost.

In this work, we examine the singular values (and thus the condition number) of the random feature matrix. In particular, consider the following formulation of the random feature regression problem. Let $f$ be the target function and suppose we are given samples $\left(\bx_j, y_j \right)$ where $y_j =f(\bx_j)+e_j$ with noise $e_j$ for $j\in[m]$. We approximate the target function by the finite sum
\begin{equation*}
    f^\sharp(\bx) = \sum_{k=1}^N c^\sharp_k \,\phi(\bx,\boldsymbol{\omega}_k).
\end{equation*}
where $\phi(\bx,\boldsymbol{\omega}) = \phi(\langle \bx,\boldsymbol{\omega} \rangle)$
is a feature map parameterized by a random variable $\boldsymbol{\omega}$ drawn from some user defined probability $\rho(\omega)$.  The function $\phi$ is an activation function and is often chosen to be the complex exponential function or ReLU.  In this way, a random feature model is a two-layer neural network where the hidden weight layer is randomized by some prescribed process and not trained.  If we define the output vector $\boldsymbol{y} = [y_1,\dots,y_m]^T\in\mathbb{C}^m$ and the random feature matrix $\A\in\mathbb{C}^{m\times N}$ whose elements are defined as $\A_{j,k}=\phi(\langle \bx_j,\boldsymbol{\omega}_k \rangle)$ for $j\in[m]$ and $k\in[N]$, then the random feature ridge regression problem is
\begin{equation*}
\min_{\boldsymbol{c}\in\mathbb{C}^{N}}  \ \|\boldsymbol{y}-\boldsymbol{A}\boldsymbol{c}\|_{2}^2 + \lambda \|\boldsymbol{c}\|_2^2
\end{equation*}
for ridge parameter $\lambda\geq0$. In the case where $\lambda=0$ we refer to the problem as \textit{ridgeless} regression. The solution to the regression problem and its generalization error depends on the spectrum of the Gram matrix $\boldsymbol{A}^*\boldsymbol{A}$, which is also the kernel matrix, see for example \cite{avron2017random, rudi2017generalization, li2019towards,liu2020random, chen2021conditioning}.

The earlier results on spectra of kernel matrices, at least from the random matrix theory perspective, focused on square inner-product kernels that take the form $K_{i,j} = \phi(\langle \bx_i, \bx_j \rangle)$
and thus only depend on one random variable. In \cite{el2010spectrum}, the spectra of kernel matrices that depend on $\langle \bx_i, \bx_j \rangle$ or $\|\bx_i- \bx_j\|^2$ and are consistent with their linearization were shown to converge to a Marchenko-Pastur distribution. The limiting distribution of square inner-product kernels in the high-dimensional and large dataset size setting were derived in \cite{cheng2013spectrum}. The convergence of the spectral norm and the extreme eigenvalues of these matrices were determined in \cite{fan2019spectral}, specifically, that the spectral norm converges almost surely to the edge of the limiting spectrum.

In this work, we consider the concentration of the spectrum of asymmetric rectangular (nonlinear) random matrices $\A_{j,k}=\phi(\langle \bx_j, \boldsymbol{\omega}_k \rangle)$ which depends on two variables, motivated by the random matrices used in random feature models. Specifically, these matrices are the (non-fixed) dictionaries that are obtained by randomizing and not training hidden layers in a neural network. The spectrum of the random feature matrices, in particular the spread of the singular values, can be used to understand the numerical and theoretical properties of various random feature algorithms. Recent results in the literature consider the asymptotic distribution of the the Gram matrix formed from this type of asymmetric rectangular random matrix. The limiting spectral density of $\A^T\A$ when the samples and weights are Gaussian was computed  in \cite{pennington2017nonlinear} using the moment method and in \cite{hastie2019surprises} using an approach from \cite{cheng2013spectrum}, see also the related work in \cite{pastur2020random, pastur2020random2}.
The authors of \cite{benigni2019eigenvalue} extended the previous results to the case where  $\bx$ and $\boldsymbol{\omega}$ are both subgaussian. In the setting where $\bx$ is deterministic and $\boldsymbol{\omega}$ is random, \cite{louart2018random} determined the empirical spectrum in the large data and dimension limit. In \cite{liao2020random}, the authors provide a precise characterization of the Gram matrix in the setting where the dimension, the size of the feature space, and  the size of the dataset are large and comparable. In particular, in this setting they showed that the limit of the Gram matrix generated from random Fourier features is not the Gaussian kernel. In \cite{wang2021deformed}, non-asymptotic estimates are given in the overparameterized setting using a pairwise approximate orthogonality condition \cite{fan2020spectra}. Also, \cite{ozccelikkale2020sparse} considers the behavior of the random Fourier features model in high dimensions. For the setting with random data and fixed frequencies, Theorem 6.1 in \cite{ozccelikkale2020sparse} shows high probability bounds similar to our Theorem 4.1.  The article uses the bounds to characterize performance for sparse recovery with Fourier features, which we became aware of after publication of this work.

Random feature regression also exhibits the double descent phenomena \cite{mei2019generalization, liao2020random, chen2021conditioning}, in which the risk is low in the underparameterized and overparameterized regions but peaks at the interpolation threshold \cite{belkin2018understand, belkin2019reconciling, belkin2020two}. This behavior is intrinsically dependent on the characterization of the spectrum of the random feature matrix. In \cite{mei2019generalization},  a detailed analysis of the double descent behavior of the random feature regression problem is shown, with the assumption that one has $m$ data samples from the $d$-dimensional sphere $\mathbb{S}^{d-1}$ and $N$ random features, with $N,m, d \rightarrow \infty$ but comparable. In particular, they showed that overparameterization is necessary to obtain the optimal test error in certain settings. In \cite{chen2021conditioning}, the singular values of the random feature matrix with Gaussian data samples and weights are shown to concentration around $1$ with high probability when the complexity ratio $\frac{N}{m}$ scales like $\log^{-1}(N)$ (underparameterized) or $\log(m)$ (overparameterized). In addition, they showed that the condition number becomes unbounded for $N$ close to $m$ (i.e. a double descent phenomena for the condition number), which also provided a mechanism for the double descent in the generalization error associated with ridgeless random feature regression. 
 
\subsection{Our Contributions}

In this work, we derive concentration bounds on the spectrum of asymmetric rectangular (nonlinear) random matrices whose entries are of the form $\A_{j,k}=\phi(\langle \bx_j, \boldsymbol{\omega}_k \rangle)$. Throughout this work, we set the activation function to be the complex exponential, i.e. $\phi(z) = \exp(iz)$. We expect similar results to hold for other activation functions.

We consider two settings on the two variables, either they are both random variables (like in \cite{pennington2017nonlinear, hastie2019surprises,mei2019generalization, pastur2020random, pastur2020random2,chen2021conditioning}) or one is a random variable and the other is \textit{well-separated} (see Section \ref{sec:theory} for the precise statement). 
Similar to \cite{chen2021conditioning}, we focus on the finite $m$ and $N$ setting; however, in this work we provide a new characterization for the spectrum as a function of the parameters. Our results improve and generalize \cite{chen2021conditioning} by only requiring one of the variables to be Gaussian and separating the variance and dimensional parameters in the theory. Our results also complement the previous work in the literature highlighted in the introduction, for example, by considering the dimensional scaling in the sampling process \cite{mei2019generalization}, subgaussian random variables \cite{benigni2019eigenvalue}, and incorporating more general data sampling processes. 

Rectangular matrices whose entries are sampled i.i.d. from a Gaussian (or subgaussian) distribution are close to isometries when the ratio of the number of rows and columns scale logarithmically. Intuitively, this should imply that a random feature matrix built from a random rectangular matrix (i.e. the random weights) should be well-conditioned if the dimension scales like the number of features and log factors. We show that the conditional expectation of the random feature matrix concentrates quickly to the expectation as the dimension of the input increases. Thus, one can relax conditions in previous works when the dimension is sufficiently large (this is made precise in the theorem statements). One reason this is important is that these results provide more practical parameter regimes in which one should expect (with high-probability) to obtain well-conditioned linear systems and trained models that generalize to new data. Our results hold for a finite number of data samples and weights, which differs from the asymptotic analysis provided in other works.



\subsection{Notation}
Let $i=\sqrt{-1}$ be the imaginary unit. For an integer $N$, the set $[N]$ is defined as $[N]=\{1,2,\dots,N\}$. We use bold letters to denote vectors and matrices. The $d \times d$ identity matrix is denoted $\boldsymbol{I}_d$. We denote the $\ell^p$-norm of a vector $\bx\in\mathbb{C}^d$ by $\|\boldsymbol{x}\|_p$ and the induced $\ell^p$ norm of a matrix $\A\in\mathbb{C}^{m\times N}$ by $\|\A\|_p$. The transpose of a matrix $\A\in\mathbb{C}^{m\times N}$ is denoted by $\A^T$, and the conjugate transpose is denoted by $\A^*$. We use $\mathcal{N}(\boldsymbol{\mu},\boldsymbol{\Sigma})$ to denote the Gaussian distribution with mean vector $\boldsymbol{\mu}$ and covariance matrix $\boldsymbol{\Sigma}$.

\section{Summary of Main Results}

Let $\{\bx_j\}_{j\in[m]}\subset\mathbb{R}^d$ be data points sampled from a distribution $\mu$ and $\{\bomega_k\}_{k\in[N]}\subset\mathbb{R}^d$ be feature weights sampled from another distribution $\rho$. Define the random feature matrix $\A$ component-wise by $\A_{j,k} = \exp(i\langle \bx_j,\bomega_k \rangle)$. In this section, we give the conditions on the dimension $d$, the number of data points $m$, and the number of feature weights $N$ so that the normalized Gram matrix $m^{-1}\A^*\A$ (or $N^{-1}\A\A^*$) is close to both the identity matrix $\boldsymbol{I}$ and its expectation.

\subsection{Concentration with Randomized Inputs}
We consider the setting where the covariance matrix of $\mu$ is a multiple of $\frac{1}{d}\boldsymbol{I}_d$ so that the expectation of $\|\bx_j\|_2$ is a constant independent of $d$ and thus the data concentration are the sphere with a fixed radius. This is a weaker version of the assumptions used in \cite{mei2019generalization}. The motivation for this is to avoid the data becoming arbitrary large in high-dimensions. The main theorems when both $\bx$ and $\bomega$ are random variables are stated below.

\begin{theorem}[Concentration in the Underparameterized Setting]\label{thrm:MainUnder}
Suppose that the random feature matrix $\A$ is defined component-wise by $\A_{j,k} = \exp(i\langle \bx_j,\bomega_k \rangle)$, $\{\bx_j\}_{j\in[m]}\subset\mathbb{R}^d$ are data points sampled from $\mathcal{N}(\boldsymbol{0},\frac{\gamma^2}{d}\boldsymbol{I})$, and $\{\bomega_k\}_{k\in[N]}\subset\mathbb{R}^d$ are feature weights such that the components of $\bomega_k$ are independent mean-zero subgaussian random variables with the same variance $\sigma^2$ and the same subgaussian parameters $\beta,\kappa$ (see Section~\ref{sec:subgaussian}). Then there exist a constant $C_1>0$ (depending only on the subgaussian parameters) and a universal constant $C_2>0$ such that if the following conditions hold
\begin{align}
    &d\geq C_1 \log\left(\frac{N}{\delta}\right) \label{eq: condition of d (Under)}\\
   & \gamma^2\sigma^2\geq 4\log\left(\frac{2N}{\eta}\right) \label{eq: condition of variances (Under)}\\
  &m\geq C_2\eta^{-2} N\log\left(\frac{2N}{\delta}\right),\label{eq:mlogN}
\end{align}
for $\delta,\eta\in(0,1)$, then we have
\begin{align}
    \left\| \frac{1}{m}\A^*\A-\boldsymbol{I}_N \right\|_2 \leq 2\eta, \label{eq: Gram-I (Under)}
\end{align}
with probability at least $1-3\delta$. Moreover, if $\eta\geq 2\delta$ (which holds for practical $\eta$ and $\delta$), then we simultaneously have
\begin{align}
    \left\| \frac{1}{m}\A^*\A-\mathbb{E}_{\bx,\bomega}\left[\frac{1}{m}\A^*\A\right] \right\|_2\leq 2\eta \label{eq: Gram-Expectation (Under)}.
\end{align}
\end{theorem}
The main idea in the proof of Theorem \ref{thrm:MainUnder} is to bound the difference in \eqref{eq: Gram-I (Under)} by
\begin{align}
    \left\| \frac{1}{m}\A^*\A-\boldsymbol{I}_N \right\|_2\leq \frac{1}{m}\left\|\A^*\A - \mathbb{E}_{\bx}\left[\A^*\A \right] \right\|_2 + \left\| \mathbb{E}_{\bx}\left[\frac{1}{m}\A^*\A\right] - \boldsymbol{I}_N \right\|_2 \label{eq: break the 2-norm of Gram-I},
\end{align}
and the difference in \eqref{eq: Gram-Expectation (Under)} by
\begin{align}
    \left\| \frac{1}{m}\A^*\A-\mathbb{E}_{\bx,\bomega}\left[\frac{1}{m}\A^*\A\right] \right\|_2\leq \frac{1}{m}\left\|\A^*\A - \mathbb{E}_{\bx}\left[\A^*\A \right] \right\|_2 + \frac{1}{m}\left\| \mathbb{E}_{\bx}\left[\A^*\A\right] - \mathbb{E}_{\bx,\bomega}[\A^*\A] \right\|_2\label{eq: break the 2-norm of Gram-Expectation}.
\end{align}
While entries of these matrices are not i.i.d., the first term on the right-hand side of \eqref{eq: break the 2-norm of Gram-I} (and \eqref{eq: break the 2-norm of Gram-Expectation}) can be decomposed as the summation of independent matrices which allows us to leverage stronger matrix concentration inequalities. The remaining terms use a weaker result on large deviations which, surprisingly, does not change the overall complexity bounds. In the proofs, we will show that each of the terms on the right-hand side of \eqref{eq: break the 2-norm of Gram-I} and \eqref{eq: break the 2-norm of Gram-Expectation} are bounded by $\eta$ simultaneously when the conditions \eqref{eq: condition of d (Under)}, \eqref{eq: condition of variances (Under)} and \eqref{eq:mlogN} are satisfied (see Section \ref{sec:theory}). Condition  \eqref{eq: condition of d (Under)} ensures that the dimension is large enough so that points separate and thus provides a minimal distance condition between random weight vectors in high dimensions.  Condition  \eqref{eq: condition of variances (Under)} resembles an uncertainty principle between the spread of the samples and the weights. Lastly, condition \eqref{eq:mlogN} is a complexity relation between the number of samples and the number of features. Similar conditions are imposed in the other theorems.

By the symmetry of the input variables, we also have the bounds for the overparameterized case $m<N$.

\begin{theorem}[Concentration in the Overparameterized Setting]\label{thrm:MainOver}
Suppose that the random feature matrix $\A$ is defined component-wise by $\A_{j,k} = \exp(i\langle \bx_j,\bomega_k \rangle)$, $\{\bomega_k\}_{k\in[N]}\subset\mathbb{R}^d$ are feature weights sampled from $\mathcal{N}(\boldsymbol{0},\sigma^2\boldsymbol{I}_d)$, and $\{\bx_j\}_{j\in[m]}\subset\mathbb{R}^d$ are data points such that the components of $\bx_j$ are independent mean-zero subgaussian random variables with the same variance $\gamma^2/d$ and the same subgaussian parameters $\beta,\kappa$. Then there exist a constant $C_1>0$ (depending only on subgaussian parameters) and a universal constant $C_2>0$ such that if the following conditions hold
\begin{align*}
    &d\geq C_1 \log\left(\frac{m}{\delta}\right) \\
   & \gamma^2\sigma^2\geq 4\log\left(\frac{2m}{\eta}\right)\\
  &N\geq C_2\eta^{-2} m\log\left(\frac{2m}{\delta}\right),
\end{align*}
for $\delta,\eta\in(0,1)$, then we have
\begin{align*}
    \left\| \frac{1}{N}\A\A^*-\boldsymbol{I}_m \right\|_2 \leq 2\eta ,
\end{align*}
with probability at least $1-3\delta$. Moreover, if $\eta\geq 2\delta$ (which holds for practical $\eta$ and $\delta$), then we simultaneously have
\begin{align*}
    \left\| \frac{1}{N}\A\A^*-\mathbb{E}_{\bx,\bomega}\left[ \frac{1}{N}\A\A^* \right] \right\|_2 \leq 2\eta.
\end{align*}
\end{theorem}

A direct consequence of Theorem \ref{thrm:MainUnder} is that all of the eigenvalues of $m^{-1}\A^*\A$ are close to $1$. Specifically, if the conditions in Theorem \ref{thrm:MainUnder} are satisfied, then
\begin{align*}
    \left|\lambda_k\left( \frac{1}{m}\A^*\A \right)-1\right|\leq 2\eta,
\end{align*}
with probability at least $1-3\delta$. Here $\lambda_k(\boldsymbol{B})$ is the $k$-th eigenvalue of the matrix $\boldsymbol{B}$. Similar results also hold for $N^{-1}\A\A^*$ if the conditions in Theorem \ref{thrm:MainOver} are satisfied. This provides an upper bound for the largest eigenvalue and a lower bound for the smallest eigenvalue. Therefore, we can conclude that the matrix $\A$ has small condition number when the complexity conditions in the theorems are satisfied. 

\subsection{Concentration to the Gaussian Kernel}
Theorem~\ref{thrm:MainOver} is actually a consequence of the more general results which only requires that the data is well-separated, i.e. there is a minimum distance between data samples. 
\begin{theorem}[Concentration to the Kernel]\label{Thm: Concentration to kernel}
Suppose that the random feature matrix $\A$ is defined (component-wise) by $\A_{j,k} = \exp(i\langle \bx_j,\bomega_k \rangle)$, the feature weights $\{\bomega_k\}_{k\in[N]}\subset\mathbb{R}^d$ are sampled from $\mathcal{N}(\boldsymbol{0},\sigma^2\boldsymbol{I}_d)$, and that for the data points $\{\bx_j\}_{j\in[m]}\subset\mathbb{R}^d$ there is a constant $R>0$ such that $\|\bx_j-\bx_k\|_2^2\geq R$ for all $j,k\in[m]$ with $j\neq k$. If the following conditions hold
\begin{align}
    N\geq C\eta^{-2}m\log\left(\frac{2m}{\delta}\right) \label{eq: condition of number of samples thrm3} \\
    \sigma^2 \geq \frac{2}{R}\log\left(\frac{m}{\eta}\right),\label{eq: condition of data variance thrm3}
\end{align}
for some $\delta,\eta\in(0,1)$, where $C>0$ is a universal constant. Then with probability at least $1-\delta$ we have
\begin{align*}
    \left\|\frac{1}{N}\bA\bA^* - \mathbb{E}_{\bomega}\left[\frac{1}{N}\bA\bA^*\right]\right\|_2 \leq  \eta.
\end{align*}
\end{theorem}
The term $\mathbb{E}_{\bomega}\left[\frac{1}{N}\bA\bA^*\right]$ is in fact the associated kernel matrix for this problem, i.e. the Gaussian kernel matrix.
Although conditions \eqref{eq: condition of number of samples thrm3} and \eqref{eq: condition of data variance thrm3} do not directly use the dimension of the data, the results implicitly improve in high-dimensions where the distance between data points can be larger.

\subsection{Comparison with Similar Results}
There are several important differences between Theorems \ref{thrm:MainUnder} and \ref{thrm:MainOver} as compared to Theorem 3.1(a-b) from \cite{chen2021conditioning}. For this discussion, we will ignore the universal constants since they may not be optimal in the theorems discussed (and come from difference sampling distributions) and instead focus on the parameter dependencies. The first difference is the dependence between the failure rate $\delta$, the concentration parameter $\eta$, and the number of random features $N$. Theorems \ref{thrm:MainUnder} and \ref{thrm:MainOver} uncovers a refined relationship between the parameters when the dimension is sufficiently large, in particular, $N$ is controlled by
 $$\min\left(\delta \exp\left(d\right), \eta \exp\left(\gamma^2\sigma^2 \right) \right)$$
 i.e. larger dimensions lead to a smaller failure rate,
while in \cite{chen2021conditioning} (when $d$ is sufficiently large) $N$ is controlled by
$$\sqrt{\delta} \eta \exp\left(\gamma^2\sigma^2 \right)$$
which implies that the product of the variances must be used to compensate for the failure rate and concentration parameter. We also extend the results to include subgaussian sampling or weights, which only change the universal constants in our bounds. In addition, we show that the concentration relies on the separation between points and does not require both variables to be random (see Theorem \ref{Thm: Concentration to kernel}).

While the theorems and proofs in Section \ref{sec:theory} use similar concentration techniques to those found in the compressive sensing (CS) literature, they differ in several key places. Both consider rectangular nonlinear random matrices; however, in this work the two inputs for the matrix can both be random variables.  The bounded orthonormal systems (BOS) matrices found in compressed sensing assumes that basis parameters ($\bomega$ for example) are fixed \cite{foucart13}. Some standard examples of BOS include the Fourier basis on $[0,1]$ where ${\bx}$ is sampled from the uniform distribution on $[0,1]$, the tensorized Legendre polynomial basis where ${\bx}$ is sampled from the uniform distribution on $[-1,1]^d$, or the multivariate Hermite polynomials where ${\bx}$ is sampled from the multivariate Gaussian distribution in $d$-dimensions \cite{rauhut2016interpolation}. In each of these cases, the basis is generated so that it is orthogonal with respect to the sampling density for ${\bx}$ which is not the case for random feature matrices. In addition, our results hold for unbounded data (e.g. Gaussian), which has been a long-standing question for the analysis of the restricted isometry property for BOS matrices \cite{gilbert2019sparse}. This is an example where the random feature matrices are more robust to the data sampling than orthogonality-based methods. 




\section{Numerical Experiments}

In this section, we verify some of the results numerically and show that the condition on the dimension, i.e. $d\geq C_1 \log\left(\frac{N}{\delta}\right)$, may have a favorable constant in practice. In Figure~\ref{figure:VsDim}, we display the concentration of the singular values as a function of the dimension by plotting the maximum and minimum singular values. The curves envelop the range of the singular values. The random feature matrix is constructed using the complex exponential function with $m=100$ random samples drawn from the normal distribution with $\gamma=1$ and $N=5000$ random weights. The matrices are normalized so that each column has unit $\ell^2$ norm. We chose the scaling between $m$ and $N$ so that Equation \eqref{eq:mlogN} would hold. The weights are drawn from the normal distribution with the standard deviation $\sigma$ specified in Figure~\ref{figure:VsDim}. For each dimension, we used 10 trials to calculate the mean of the extreme singular values (the solid curves) and one standard deviation (the shaded regions). The plots indicate that as $\sigma$ increases or as the dimension $d$ increases, the singular values concentrate quickly. In this case as $\sigma$ increases, the potential range decreases up to the range indicated by $\sigma=3$, i.e. all pairs of curves for $\sigma\geq3$ will have the same upper and the same lower plateaus as the curves generated with $\sigma=3$. Additionally, for $\sigma=3$, once the dimension exceeds $3$ the matrix is well-conditioned. 

\begin{figure}[t!]
\begin{center}
\includegraphics[trim={0.1cm 0.1cm 0.1cm 0.1cm}, width=4in, clip]{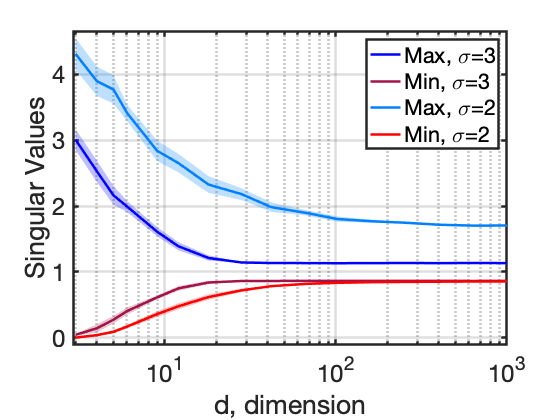}
\caption{\textbf{Extreme Singular Values versus dimension}: The plot displays the maximum and minimum singular values for various dimensions $d$. For each dimension, 10 trials are used to calculate the mean value (the solid curves) and one standard deviation (shaded regions). The random feature matrix is the complex exponential with $m=100$, $N=5000$, and $\gamma=1$. The standard deviation of the weights $\sigma$ are specified for each curve. Even with a moderate dimension, i.e. $d\geq10$, the condition number is already less than $15$ for both examples.  } \label{figure:VsDim}
\end{center}
\end{figure}

\begin{figure}[h!]
\begin{center}
\includegraphics[trim={0.1cm 0.1cm 0.1cm 0.1cm}, width=2.75in, clip]{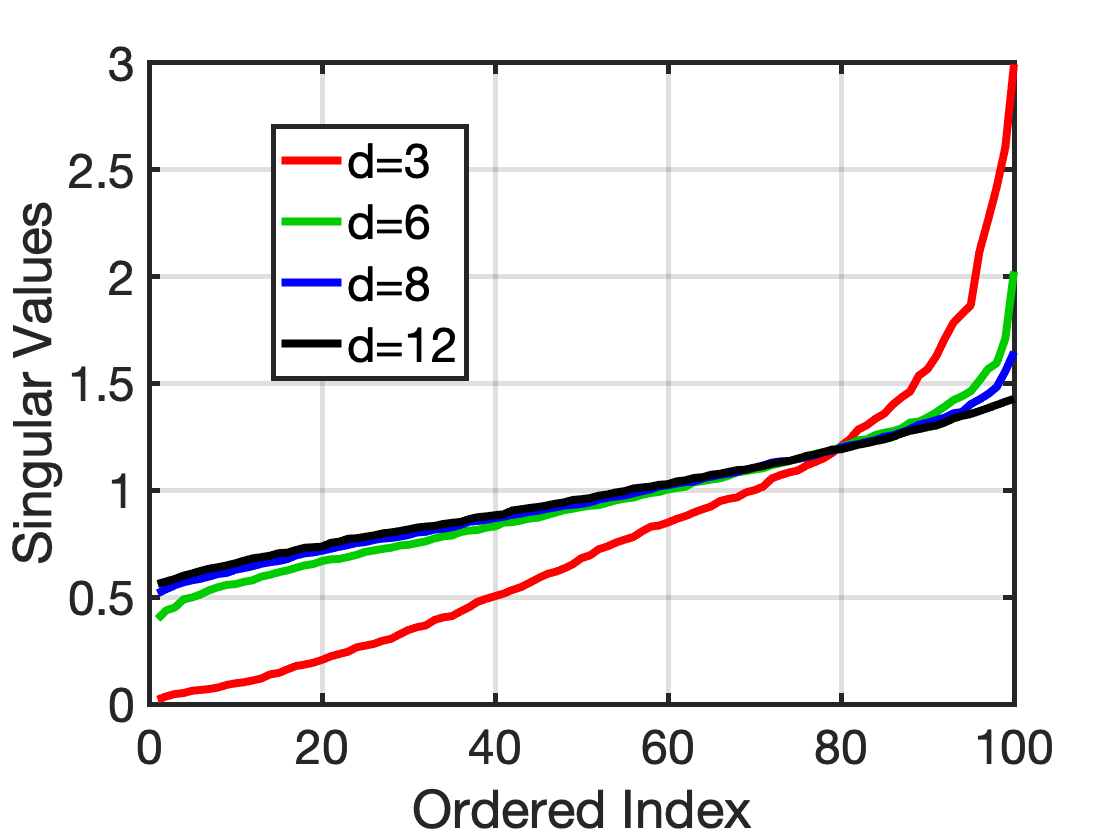}
\includegraphics[trim={0.1cm 0.1cm 0.1cm 0.1cm}, width=2.75in, clip]{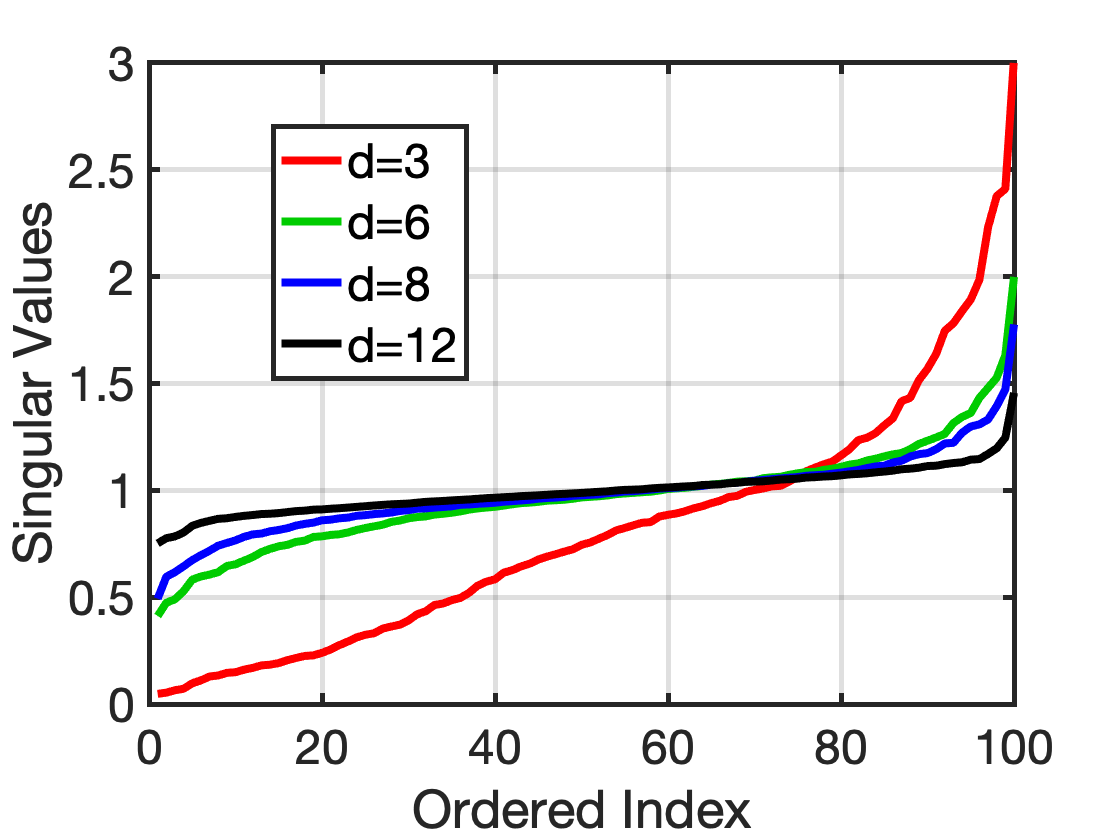}
\caption{\textbf{Singular Value Distribution with Different $d$ and $N$}: The figure shows the distribution of the singular values in ascending as a function of the dimension $d$ and the number of random features $N$.  The random feature matrix is the complex exponential with $m=100$, $\gamma=1$, and $\sigma=3$. The plot on the left uses $N=500$ and the plot on the right uses $N=5000$. This experiment shows that the singular values concentration around $1$ quickly in dimension and $N$.} \label{fig:svd_dist1}
\end{center}
\end{figure}

\begin{figure}[h!]
\begin{center}
\includegraphics[trim={0.1cm 0.1cm 0.1cm 0.1cm}, width=2.75in, clip]{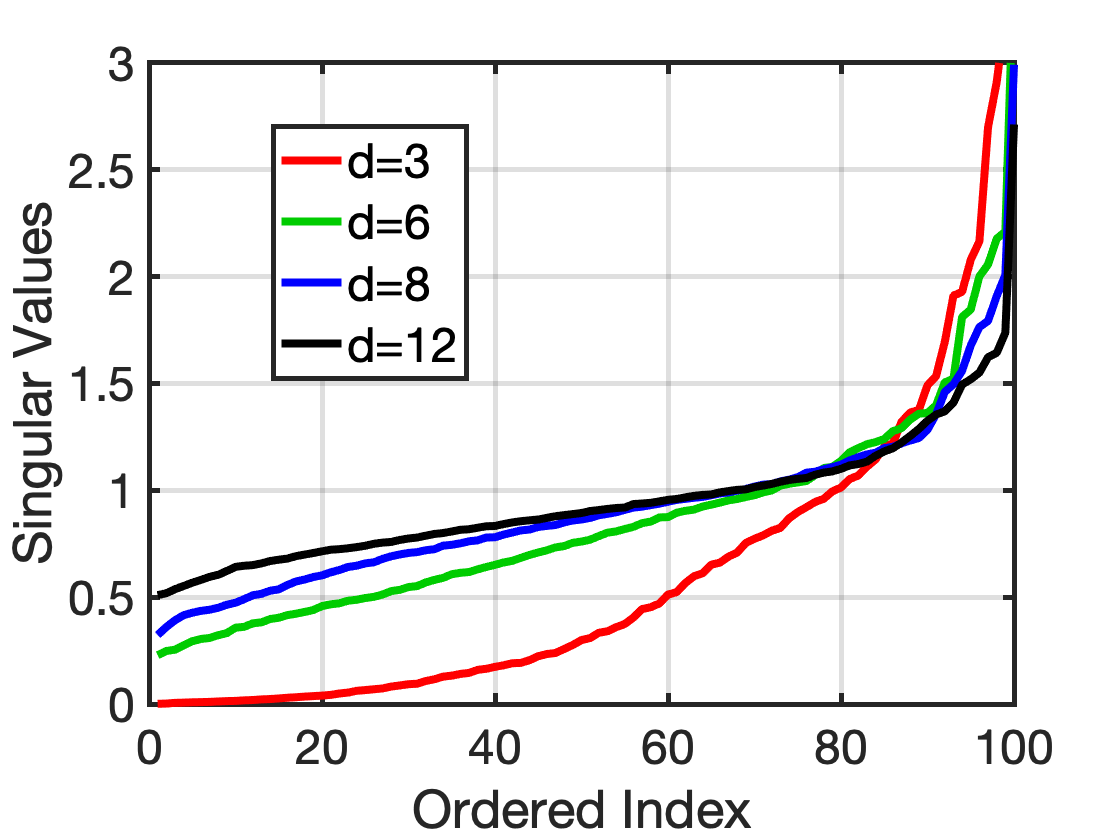}
\includegraphics[trim={0.1cm 0.1cm 0.1cm 0.1cm}, width=2.75in, clip]{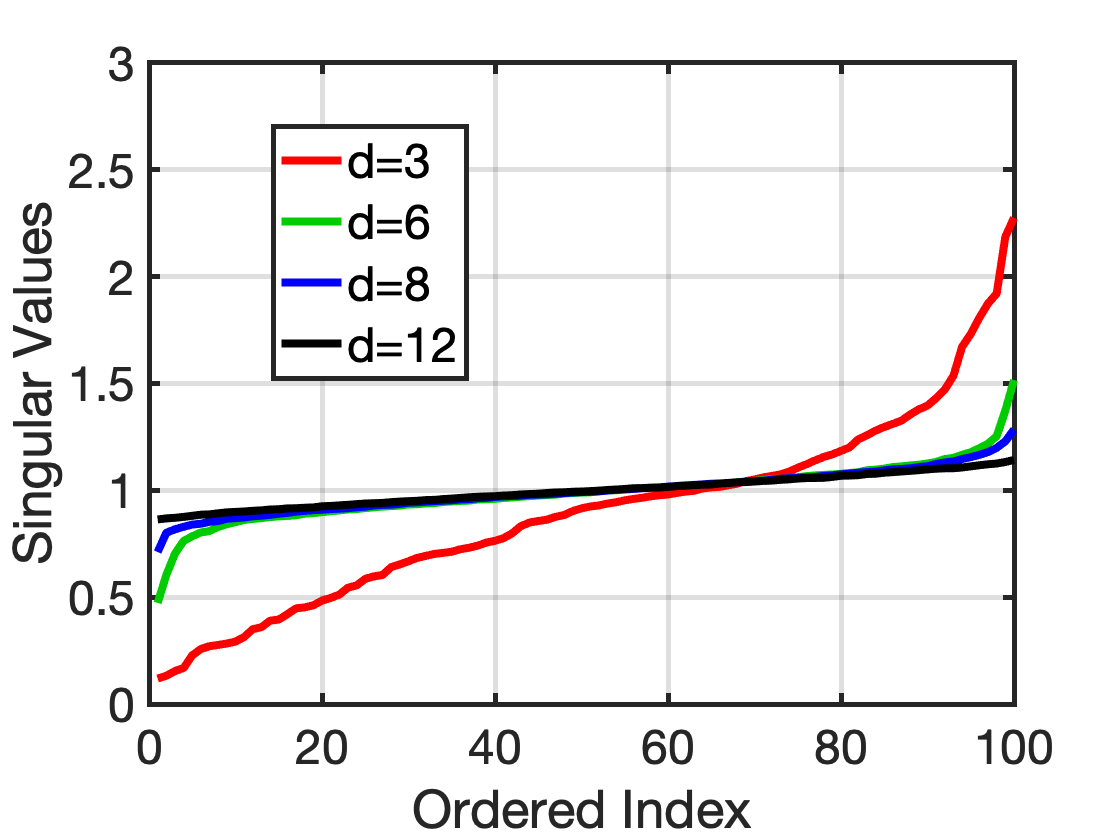}
\caption{\textbf{Singular Value Distribution with Different $d$ and $\sigma$}: The figure shows the distribution of the singular values in ascending as a function of the dimension $d$ and the standard deviation of the random weights $\sigma$.  The random feature matrix is the complex exponential with $m=100$, $N=5000$, and $\gamma=1$. The plot on the left uses $\sigma=2$ and the plot on the right uses $\sigma=4$. The plots indicate that a small increase in $\sigma$ has a dramatic effect on the concentration once $d$ is above some value.} \label{fig:svd_dist2}
\end{center}
\end{figure}

We study the distribution of the singular value for different parameters in Figure \ref{fig:svd_dist1} and Figure \ref{fig:svd_dist2}. The theory shows that the distribution should be close to the line $y=1$. In Figure \ref{fig:svd_dist1}, the dimension $d$ and the number of random features $N$ are varied and the distribution of the singular values are plotted in ascending order (thus the $x$-axis is the sorted index).  The random feature matrix is the complex exponential (with normalized columns) with $m=100$, $\gamma=1$, and $\sigma=3$. The left plot in Figure \ref{fig:svd_dist1} has $N=500$ random features and the right plot in Figure \ref{fig:svd_dist1} has $N=5000$. In both plots, when the dimension is $3$, the minimum singular values are close to zero, thus leading to ill-conditioning of $\bA$. For $d\geq 6$ the singular values range from about $0.5$ to $2$ with noticeable concentration for $d=12$. Note that the concentration is more significant for larger $N$ while retaining a similar dependency on $d$. In Figure \ref{fig:svd_dist2}, we consider the case where the dimension $d$ and the standard deviation of the random features $\sigma$ are varied. The random feature matrix uses the same setup as  Figure \ref{fig:svd_dist1} except that $N=5000$ while the standard deviation $\sigma$ is varied. The left plot in the Figure \ref{fig:svd_dist2} uses $\sigma=2$ and the right plot in Figure \ref{fig:svd_dist2} uses $\sigma=4$, showing that larger values of $\sigma$ and $d$ lead to better concentration. Note that there is an asymmetry in the distributions of the singular values around 1 as seen in Figure \ref{fig:svd_dist1} and Figure \ref{fig:svd_dist2}.




\section{Theoretical Analysis}\label{sec:theory}

In this section, we present the key results for bounding the differences in \eqref{eq: Gram-I (Under)} and \eqref{eq: Gram-Expectation (Under)}. These will then lead to the proof of Theorems \ref{thrm:MainUnder}, \ref{thrm:MainOver}, and \ref{Thm: Concentration to kernel}. The main arguments are split into multiple theorems with shorter proofs for clarity and modularity. In all of our results we assume that the random feature matrix $\A$ is defined as $\A_{j,k} = \exp(i\langle \bx_j,\bomega_k \rangle)$.


\subsection{Concentration with Separated Weights}

Using the matrix Bernstein inequality (Lemma \ref{Matrix Bernstein} in the Appendix), we show that the Gram matrix is close to its expectation for feature weights that are well-separated.
\begin{theorem}[Concentration with Separated Weights]\label{Thm: Concentration with Separated Weights}
Let $\{\bx_j\}_{j\in [m]}\subset\mathbb{R}^d$ be the data sampled from $\mathcal{N}(\boldsymbol{0},\frac{\gamma^2}{d}\bI_d)$. Suppose that $\{\bomega_k\}_{k\in[N]}\subset \mathbb{R}^d$ is a set of feature weights, and there is a constant $R>0$ such that $\|\bomega_j-\bomega_k\|_2^2\geq Rd$ for all $j,k\in[N]$ with $j\neq k$. If the following conditions hold
\begin{align}
    m\geq C\eta^{-2}N\log\left(\frac{2N}{\delta}\right) \label{eq: condition of number of samples} \\
    \gamma^2 \geq \frac{2}{R}\log\left(\frac{N}{\eta}\right)\label{eq: condition of data variance}
\end{align}
for some $\delta,\eta\in(0,1)$, where $C>0$ is a universal constant. Then with probability at least $1-\delta$ we have
\begin{align*}
    \left\|\frac{1}{m}\bA^*\bA - \mathbb{E}_{\bx}\left[\frac{1}{m}\bA^*\bA\right]\right\|_2 \leq  \eta.
\end{align*}
\end{theorem}
\begin{proof}
Let $\bX_\ell$ be $\ell$-th column of $\bA^*$. Defined the random matrices $\{\boldsymbol{Y}_\ell\}_{\ell\in[m]}$ as 
\begin{align*}
    \boldsymbol{Y}_{\ell} = \bX_\ell \bX_\ell^* - \mathbb{E}_{\bx}[\bX_\ell \bX_\ell^*].
\end{align*}
Then $(\boldsymbol{Y}_\ell)_{j,j}=0$ and $(\boldsymbol{Y}_\ell)_{j,k} = \exp(i\langle \bx_\ell, \bomega_k-\bomega_j\rangle) - \exp(-\gamma^2\|\bomega_k-\bomega_j\|_2^2/(2d))$ for $j,k\in[N]$ with $j\neq k$. Note that $\boldsymbol{Y}_\ell$ is self-adjoint and its induced $\ell^2$ norm is bounded by its largest eigenvalue. By Gershgorin's disk theorem and condition \eqref{eq: condition of data variance},
\begin{align*}
    \|\boldsymbol{Y}_\ell\|_2 \leq \max_{j\in[N]} \sum_{k\neq j} \left|e^{i\langle \bx_\ell, \bomega_k-\bomega_j \rangle} - e^{-\frac{\gamma^2}{2d}\|\bomega_k-\bomega_j\|_2}\right| \leq N\left(1+e^{-\frac{\gamma^2}{2}R}\right)\leq N+\eta.
\end{align*}
The variance parameter in Lemma \ref{Matrix Bernstein} is bounded by
\begin{align*}
    \left\| \sum_{\ell=1}^m \mathbb{E}_{\bx}[\boldsymbol{Y}_\ell^2] \right\|_2 \leq \sum_{\ell=1}^m \|\mathbb{E}_{\bx}[\boldsymbol{Y}_\ell^2]\|_2 
    &=\sum_{\ell=1}^m \|N \mathbb{E}_{\bx}[\bX_\ell\bX_\ell^*] - (\mathbb{E}_{\bx}[\bX_\ell\bX_\ell^*])^2\|_2 \\
    &\leq m [N(1+\eta)+(1+\eta)^2].
\end{align*}
Here we use the fact that $\bX_\ell$ is a vector with $\|\bX_\ell\|_2=\sqrt{N}$, which implies $\bX_\ell \bX_\ell^* \bX_\ell \bX_\ell^* = N \bX_\ell \bX_\ell^*$, and $\mathbb{E}_{\bx}[\bX_\ell\bX_\ell^*]$ is self-adjoint whose $\ell^2$ norm is bounded by
\begin{align*}
    \|\mathbb{E}_{\bx}[\bX_\ell \bX_\ell^*]\|_2 \leq 1+ \max_{j\in[N]} \sum_{k\neq j}\left|e^{-\frac{\gamma^2}{2d}\|\bomega_k-\bomega_j\|_2^2}\right|\leq 1+N\exp\left(-\frac{\gamma^2}{2}R\right)\leq 1+\eta,
\end{align*}
by Gershgorin's disk theorem. Since $\{\boldsymbol{Y}_\ell\}_{\ell\in[m]}$ are independent mean-zero self-adjoint matrices, applying Lemma \ref{Matrix Bernstein} with $K=N+\eta$ and $\sigma^2=m[N(1+\eta)+(1+\eta)^2]$ then gives
\begin{align*}
    \mathbb{P}\left(\left\|\frac{1}{m}\bA^*\bA-\mathbb{E}_{\bx}\left[\frac{1}{m}\bA^*\bA\right]\right\|_2\geq \eta\right) &= \mathbb{P}\left(\left\|\sum_{\ell=1}^m \boldsymbol{Y}_\ell\right\|_2\geq m\eta\right) \\
    &\leq 2N \exp\left( -\frac{m\eta^2/2}{N(1+\eta)+(1+\eta)^2+(N+\eta)\eta/3}\right) \\
    &\leq 2N\exp\left(-\frac{m\eta^2}{5N+9}\right).
\end{align*}
The left-hand term is less than $\delta$, provided condition \eqref{eq: condition of number of samples} is satisfied with $C=6$ (assuming that $N\geq 9$ and $\eta<1$). This completes the proof.
\end{proof}

\subsection{Separation of Subgaussian Weights} \label{sec:subgaussian}
Theorem \ref{Thm: Concentration with Separated Weights} in the previous section requires that the weights $\{\bomega_k\}_{k\in[N]}$ are sufficiently separated, but does not place a restriction on the sampling process. Next, we show that if $\{\bomega_k\}_{k\in[N]}$ are sampled independently from a subgaussian distribution then they are separated with high probability. Recall that a random variable $X$ is called subgaussian if there exist $\beta,\kappa>0$ such that
\begin{align*}
    \mathbb{P}(|X|\geq t)\leq \beta e^{-\kappa t^2}\quad \text{for all }t>0.
\end{align*}
We call a random vector $\bX=(X_1,\dots,X_d)\in\mathbb{R}^d$ subgaussian if $X_i$ are mean-zero independent subgaussian with the same subgaussian parameters. Using a concentration inequality for $\ell^2$ norm of subgaussian vectors (Lemma \ref{Concentration of norm of subgaussian vector} in the Appendix), we have the following result.
\begin{theorem}[Separation of Subgaussian Weights]\label{Thm: Separation of Subgaussian Weights}
Suppose $\{\bomega_k\}_{k\in[N]}\subset\mathbb{R}^d$ is a set of random vectors such that the components of $\bomega_{k}$ are independent mean-zero subgaussian random variables with variance $1$ and the same subgaussian parameters $\beta,\kappa$. If the dimension $d$ satisfies the following condition
\begin{align*}
    d\geq Ct^{-2}\log\left(\frac{N}{\delta}\right),
\end{align*}
for $\delta, t\in(0,1)$, where $C>0$ is a constant depends on the subgaussian parameters, then
\begin{align*}
    \|\bomega_j-\bomega_k\|_2^2\geq (2-2t)d
\end{align*}
with probability at least $1-\delta$.
\end{theorem}
\begin{proof}
We use a result for subgaussian matrices to estimate the squared distance between $\bomega_j$ and $\bomega_k$. Denote by $\boldsymbol{W}$ the matrix which has $\bomega_j/\sqrt{d}$ as its $j$-th column. The $s$-th restricted isometry property (RIP) constant $\delta_s=\delta_s(\boldsymbol{W})$ of the matrix $\boldsymbol{W}$ is the smallest $\Delta\geq 0$ such that
$$
(1-\Delta)\|\bx\|_2^2 \leq \|\boldsymbol{W}\bx\|_2^2\leq (1+\Delta) \|\bx\|_2^2
$$
holds for all $s$-sparse vector $\bx$, i.e. $\bx$ which has at most $s$ nonzero elements. Theorem 9.2 from \citep{foucart13} with $s=2$ shows that the RIP constant $\delta_2$ of $\boldsymbol{W}\in\mathbb{R}^{d\times N}$ is less than $t$ with probability at least $1-\delta$ if 
\begin{align}
    d\geq Ct^{-2}\log\left(\frac{N}{\delta}\right)\label{eq: condition for RIP},
\end{align}
for some $C>0$ which depends only on the subgaussian parameters. Note that the RIP constant $\delta_2$ can also be defined as
$$
\delta_2 := \max_{j,k\in [N],j\neq k} \|\boldsymbol{W}_{\{j,k\}}^*\boldsymbol{W}_{\{j,k\}} - \boldsymbol{I}_2\|_2,
$$
where $\boldsymbol{W}_{\{j,k\}}$ is the submatrix of $\boldsymbol{W}$ consisting of $j$-th and $k$-th column. For any $j,k\in[N], j\neq k$,  the matrix $\boldsymbol{W}_{\{j,k\}}^*\boldsymbol{W}_{\{j,k\}} - \boldsymbol{I}_2$ takes the form
\begin{align*}
    \boldsymbol{W}_{\{j,k\}}^*\boldsymbol{W}_{\{j,k\}} - \boldsymbol{I}_2 = \frac{1}{d}
    \begin{bmatrix}
     \|\bomega_j\|_2^2 -d & \langle \bomega_k,\bomega_j \rangle \\
    \langle \bomega_j, \bomega_k \rangle & \|\bomega_k\|_2^2-d
    \end{bmatrix}.
\end{align*}
This is a symmetric matrix and its eigenvalues are
\begin{align*}
    \lambda^{\pm} = \frac{(\|\bomega_j\|_2^2+\|\bomega_k\|_2^2-2d)\pm \sqrt{(\|\bomega_j\|_2^2-\|\bomega_k\|_2^2)^2+4|\langle\bomega_j,\bomega_k \rangle|^2}}{2d}.
\end{align*}
Therefore, $\delta_2\leq t$ implies that both eigenvalues are in $[-t,t]$ and consequently,
\begin{align*}
    \frac{(\|\bomega_j\|_2^2+\|\bomega_k\|_2^2-2d)- 2|\langle\bomega_j,\bomega_k \rangle|}{2d}\geq \lambda^{-}\geq -t
\end{align*}
Thus, $\|\bomega_j-\bomega_k\|_2^2 = \|\bomega_j\|_2^2+\|\bomega_k\|_2^2 - 2\langle\bomega_j,\bomega_k \rangle \geq (2-2t)d$ for all $j,k\in [N], j\neq k$ with probability at least $1-\delta$ if condition \eqref{eq: condition for RIP} is satisfied.
\end{proof}

Theorem \ref{Thm: Separation of Subgaussian Weights} holds for arbitrary mean-zero subgaussian distribution. In particular, if $\bomega_k$ are sampled from a mean-zero Gaussian distribution, they will be separated with high probability.
\begin{corollary}[Separation of Gaussian Weights]
Suppose that $\{\bomega_k\}_{k\in[N]}\subset\mathbb{R}^d$ are sampled from $\mathcal{N}(\boldsymbol{0},\sigma^2\boldsymbol{I}_d)$. If the dimension $d$ satisfies
\begin{align*}
    d\geq Ct^{-2}\log\left(\frac{N}{\delta}\right),
\end{align*}
for $\delta,t\in(0,1)$, where $C>0$ is a universal constant. Then we have 
\begin{align*}
    \|\bomega_j-\bomega_k\|_2^2\geq (2-2t)\sigma^2d \quad \text{for all } j,k\in[N], j\neq k
\end{align*}
with probability at least $1-\delta$.
\end{corollary}

\subsection{Concentration of Random Feature Matrices with Subgassian Weights}

In the proof of Theorem \ref{Thm: Concentration with Separated Weights}, we use the condition $\|\bomega_j-\bomega_k\|_2^2\geq Rd$ to obtain an estimate of $\|\mathbb{E}_{\bx}[\bX_\ell\bX_\ell^*]\|_2$. Note that this is a self-adjoint matrix with ones along the diagonal and its off diagonal terms are relatively small when $\|\bomega_j-\bomega_k\|_2$ are sufficiently large. Therefore, if $\{\bomega_k\}_{k\in [N]}$ are sampled randomly such that each component of $\bomega_k$ is subgaussian with variance $\sigma^2$, then we would expect $\mathbb{E}_{\bx}[\A^*\A/m]$ to be close to identity $\boldsymbol{I}_N$.

\begin{theorem}[Concentration with Subgaussian Weights]\label{Thm: Concentration with Subgaussian Weights}
Suppose that $\{\bx_j\}_{j\in[m]}\subset\mathbb{R}^d$ are data points sampled from $\mathcal{N}(\boldsymbol{0},\frac{\gamma^2}{d}\boldsymbol{I}_d)$ and $\{\bomega_k\}_{k\in[N]}\subset\mathbb{R}^d$ are feature weights such that components of $\bomega_k$ are independent mean-zero subgaussian random variables with variance $\sigma^2$ and the same subgaussian parameters $\beta,\kappa$. Then there exists a constant $C>0$ (depending only on subgaussian parameters) such that if the following conditions hold:
\begin{align}
    d\geq C\log\left(\frac{N}{\delta}\right) \label{eq: condition of d}\\
    \gamma^2\sigma^2 \geq 4\log\left(\frac{2N}{\eta}\right) \label{eq: condition of variance product}
\end{align}
for some $\delta,\eta \in (0,1)$, we have
\begin{align*}
    \left\| \mathbb{E}_{\bx}\left[ \frac{1}{m}\A^*\A \right] - \boldsymbol{I}_N \right\|_2 \leq \eta 
\end{align*}
with probability at least $1-2\delta$. Moreover, if $\eta\geq 2\delta$ (which holds for practical $\eta$ and $\delta$), then we simultaneously have
\begin{align*}
    \left\| \mathbb{E}_{\bx}\left[ \frac{1}{m}\A^*\A \right] - \mathbb{E}_{\bx,\bomega}\left[ \frac{1}{m}\A^*\A \right] \right\|_2 \leq \eta.
\end{align*}
\end{theorem}
\begin{proof}
Denote by $\boldsymbol{B}$ the matrix $m^{-1}\mathbb{E}_{\bx}[\A^*\A]-\boldsymbol{I}_N$. Then $\boldsymbol{B}$ is symmetric and $(\boldsymbol{B})_{j,j}=0$, $(\boldsymbol{B})_{j,k} = \exp(-\gamma^2\|\bomega_j-\bomega_k\|_2^2/(2d))$ for all $j,k\in[N]$ with $j\neq k$. By Theorem \ref{Thm: Separation of Subgaussian Weights} with $t=3/4$, for all $j,k\in[N]$ with $j\neq k$ we have
\begin{align*}
    \|\bomega_j-\bomega_k\|_2^2\geq \frac{\sigma^2}{2}d,
\end{align*}
with probability at least $1-\delta$ if the dimension $d$ satisfies
\begin{align}
    d\geq C\log\left(\frac{N}{\delta}\right) \label{eq: condition of d for I}.
\end{align} 
Thus, each off diagonal elements is bounded (in magnitude) by 
\begin{align*}
    |(\boldsymbol{B})_{j,k}| = \exp\left( -\frac{\gamma^2}{2d}\|\bomega_j-\bomega_k\|_2^2 \right) \leq \exp\left( -\frac{\gamma^2\sigma^2}{4}\right) \quad \text{for all } j,k\in[N], j\neq k,
\end{align*}
with probability at least $1-\delta$. By Gershgorin disk theorem and condition \eqref{eq: condition of variance product}, the induced $\ell^2$ norm of $\boldsymbol{B}$ is bounded by
\begin{align*}
    \|\boldsymbol{B}\|_2\leq \max_{j\in[N]}\sum_{k\neq j} |(\boldsymbol{B})_{j,k}| \leq N\exp\left(-\frac{\gamma^2\sigma^2}{4}\right)\leq \eta,
\end{align*}
with probability at least $1-\delta$. 

Next, we denote by $\boldsymbol{C}$ the matrix $m^{-1}\mathbb{E}_{\bx}[\A^*\A]-m^{-1}\mathbb{E}_{\bx,\bomega}[\A^*\A]$. Then $\boldsymbol{C}$ is also symmetric and $(\boldsymbol{C})_{j,j}=0$, $(\boldsymbol{C})_{j,k} = \exp(-\gamma^2\|\bomega_j-\bomega_k\|_2^2/(2d)) - \mathbb{E}_{\bomega}[\exp(-\gamma^2\|\bomega_j-\bomega_k\|_2^2/(2d))]$ for all $j,k\in[N]$ with $j\neq k$. The previous argument shows that the term $\exp(-\gamma^2\|\bomega_j-\bomega_k\|_2^2/(2d))$ is small. Thus, we only need to estimate the expectation, i.e. $\mathbb{E}_{\bomega}[\exp(-\gamma^2\|\bomega_j-\bomega_k\|_2^2/(2d))]$. Since $\bomega_j$ and $\bomega_k$ are independent subgaussian vectors, $\bomega_j-\bomega_k$ is also a subgaussian vector with new parameters that depend on $\beta$ and $\kappa$. Applying Lemma \ref{Concentration of norm of subgaussian vector} yields
\begin{align*}
    \mathbb{P}(\|\bomega_j-\bomega_k\|_2^2\leq z\sigma^2 d)\leq \exp\left(-C\left(1-\frac{\sqrt{z}}{{2}}\right)^2d\right)
\end{align*}
for some constant $C>0$ which depends on the subgaussian parameters. Then by setting $z=1/2$ and decomposing the expectation (where $\chi$ is the characteristic function), we have
\begin{align*}
    &\mathbb{E}\left[\exp\left( -\frac{\gamma^2}{2d}\|\bomega_j-\bomega_k\|_2^2\right) \right]\\ 
    &\leq  \mathbb{E}\left[\exp\left( -\frac{\gamma^2}{2d}\|\bomega_j-\bomega_k\|_2^2 \right)\, {\chi}_{\|\bomega_j-\bomega_k\|_2^2 \leq \frac{\sigma^2 d}{2}}\right]  +  \mathbb{E}\left[\exp\left( -\frac{\gamma^2}{2d}\|\bomega_j-\bomega_k\|_2^2\right)\, \chi_{\|\bomega_j-\bomega_k\|_2^2 \geq \frac{\sigma^2 d}{2}} \right]  \\
    &\leq \mathbb{P}\left(\|\bomega_j-\bomega_k\|_2^2\leq \frac{\sigma^2d}{2}\right) + \exp\left(-\frac{\gamma^2\sigma^2}{4}\right) \\
    &\leq \exp(-Cd) + \exp\left(-\frac{\gamma^2\sigma^2}{4}\right),
\end{align*}
for some redefined constant $C>0$. Therefore, if 
\begin{align}
    d\geq C\log\left(\frac{2N}{\eta}\right)\label{eq: condition of d for full expectation}\\
    \gamma^2\sigma^2\geq 4\log\left(\frac{2N}{\eta}\right) \label{eq: condition of variance product for full expectation},
\end{align}
then we have the bound
\begin{align*}
    \|\boldsymbol{C}\|_2 &\leq \max_{j\in[N]}\sum_{k\neq j}|(\boldsymbol{C})_{j,k}| \\
    &\leq N \max\left\{\exp\left(-\frac{\gamma^2}{2d}\|\bomega_j-\bomega_k\|_2^2\right), \mathbb{E}_{\bomega}\left[\exp\left(-\frac{\gamma^2}{2d}\|\bomega_j-\bomega_k\|_2^2\right)\right] \right\} \leq \eta,
\end{align*}
with probability at least $1-\delta$.

Note that condition \eqref{eq: condition of d for full expectation} and condition \eqref{eq: condition of d for I} are slightly different. However, by assuming $\eta\geq 2\delta$, we can combine \eqref{eq: condition of d for full expectation} and \eqref{eq: condition of d for I} to obtain \eqref{eq: condition of d}.
\end{proof}
\begin{remark}
If $\bomega_k\sim\mathcal{N}(\boldsymbol{0},\sigma^2\boldsymbol{I})$, then the expectation $\mathbb{E}_{\bomega}[\exp(-\gamma^2\|\bomega_j-\bomega_k\|_2^2/(2d))]$ can be computed explicitly
\begin{align*}
    \mathbb{E}_{\bomega}\left[\exp\left(-\frac{\gamma^2}{2d}\|\bomega_j-\bomega_k\|_2^2\right)\right] = \left(\frac{2\gamma^2\sigma^2}{d}+1\right)^{-\frac{d}{2}},
\end{align*}
which is approximately $\exp(-\gamma^2\sigma^2)$ for large $d$. Thus, in this case, the product of $\gamma$ and $\sigma$ only needs to satisfy
$$
\gamma^2\sigma^2 \geq \log\left(\frac{N}{\eta}\right).
$$
\end{remark}

\subsection{Proof of Main Results}
Using the previous results, we can prove Theorem \ref{thrm:MainUnder}.
\begin{proof}[Proof of Theorem \ref{thrm:MainUnder}]
By Theorem \ref{Thm: Separation of Subgaussian Weights} with $t=3/4$, for weights $\{\bomega_k\}_{k\in[N]}$ whose components are independent mean-zero subgaussian random variables with variance $\sigma^2$, we have
\begin{align*}
    \|\bomega_j-\bomega_k\|_2 \geq \frac{\sigma^2d}{2},
\end{align*}
for all $j,k\in[N]$ with $j\neq k$ and with probability at least $1-\delta$ if
\begin{align*}
    d\geq C_1 \log\left( \frac{N}{\delta} \right),
\end{align*}
for some $C_1> 0$ which depends only on subgaussian parameters. Then for $\eta>0$, by Theorem \ref{Thm: Concentration with Separated Weights}, we have
\begin{align*}
    \left\| \frac{1}{m}\A^*\A - \mathbb{E}_{\bx}\left[\frac{1}{m}\A^*\A\right] \right\|_2\leq \eta,
\end{align*}
with probability at least $1-2\delta$ if
\begin{align*}
    m\geq C_2 \eta^{-2} N \log\left( \frac{2N}{\delta}\right) \\
    \gamma^2\sigma^2 \geq 4\log\left(\frac{2N}{\eta}\right),
\end{align*}
for some $C_2>0$. Lastly, by Theorem \ref{Thm: Concentration with Subgaussian Weights}, we also have
\begin{align*}
    \left\| \mathbb{E}_{\bx}\left[\frac{1}{m}\A^*\A\right]-\boldsymbol{I}_N \right\|_2\leq \eta 
\end{align*}
with probability at least $1-\delta$ if conditions \eqref{eq: condition of d (Under)} and \eqref{eq: condition of variances (Under)} are satisfied. Therefore, the difference in \eqref{eq: Gram-I (Under)} is bounded by $2\eta$ through \eqref{eq: break the 2-norm of Gram-I}. If $\eta\geq 2\delta$, then we also have
\begin{align*}
    \left\| \mathbb{E}_{\bx}\left[ \frac{1}{m}\A^*\A \right] - \mathbb{E}_{\bx,\bomega}\left[ \frac{1}{m}\A^*\A \right] \right\|_2 \leq \eta
\end{align*}
and the difference in \eqref{eq: Gram-Expectation (Under)} is bounded by $2\eta$ through \eqref{eq: break the 2-norm of Gram-Expectation}, with probability at $1-3\delta$.
\end{proof}

To prove Theorem \ref{thrm:MainOver}, consider the matrix $\A^*$ which just switches the role of $\bx$ and $\bomega$ in the theorems. The proof of Theorem~\ref{Thm: Concentration to kernel} is the same as Theorem~\ref{Thm: Concentration with Separated Weights} with $\bx$ and $\bomega$ switched and by removing the dimensional scaling in the sampling of $\bomega$ and the separation of the data samples (for consistency).

Note that in \citep{chen2021conditioning}, the union bound was used to estimate $\|\mathbb{E}_{\boldsymbol{x}}[\boldsymbol{X}_\ell\boldsymbol{X}_\ell^* - \mathbb{E}_{\boldsymbol{x},\boldsymbol{\omega}}[\boldsymbol{X}_\ell\boldsymbol{X}_\ell^*]\|_2$ which led to a condition for $N$ that depends algebraically on the probability $\delta$. By showing that i.i.d. subgaussian vectors are well-separated in high dimensions, we obtain a uniform bound which does not depend on $N$ for all entries in the matrix. Thus, the probability $\delta$ does not restrict $N$ when $d$ is large.




\section{Summary}

The main results show that the spectrum of an asymmetric rectangular (nonlinear) random matrix whose entries are of the form $\A_{j,k}=\phi(\langle \bx_j, \boldsymbol{\omega}_k \rangle)$ concentrates around its expectation and around  $1$ given particular (finite) complexity scales. We showed that this holds in the setting where both variables are random (i.e. one is normal and the other is subgaussian) and in the setting where one is a random normal variable and the other is well-separated. The conditions in the theorems relax as the dimension of the input data increases, thus showing that high-dimensional random feature matrices are well-conditioned.
In addition, in the case of subgaussian weights (or subgaussian data), we do not require that the weights (or data) follow the same distribution, as long as they have the same subgaussian parameters. This generalizes and extends the results beyond that of previous ones. The results are presented in separate parts, since they may be used for the analysis of other random feature models. For example, using these techniques, one may also be able to find the dependency of the restricted isometry property of random feature matrices on the dimension, which is useful for analyzing sparsity-based approaches for random features models \cite{yen2014sparse,hashemi2021generalization, saha2022harfe, xie2021shrimp, richardson2022sparse}.

\section*{Acknowledgement}
Z.C. and H.S. were supported in part by AFOSR MURI FA9550-21-1-0084 and NSF DMS-1752116. R.W.  was supported in part by AFOSR MURI FA9550-19-1-0005, NSF DMS 1952735, NSF HDR-1934932, and NSF 2019844.

\bibliographystyle{plain}

\appendix

\section{Useful Lemmata}

We recall the matrix Bernstein's inequality from \cite{oliveira2009concentration, tropp2012user, foucart13}, which is used to prove the main results. 

\begin{lemma}[Matrix Bernstein's inequality]\label{Matrix Bernstein}
Let $\{\bX_j\}_{j\in[m]}\subset \mathbb{C}^{N\times N}$ be independent mean-zero self-adjoint random matrices. Assume that
\begin{align*}
    \|\bX_j\|_2 \leq K\quad a.s. \text{ for all } j\in[m],
\end{align*}
and set
\begin{align*}
    \sigma^2 := \left\|\sum_{j=1}^m \mathbb{E}[\bX_j^2]\right\|_2.
\end{align*}
Then for $t>0$,
\begin{align*}
    \mathbb{P}\left(\left\|\sum_{j=1}^m \bX_j \right\|_2\geq t\right) \leq 2N\exp\left(-\frac{t^2/2}{\sigma^2+Kt/3} \right).
\end{align*}
\end{lemma}

For a random vector with subgaussian components, we have the following concentration inequality for its $\ell^2$ norm.
\begin{lemma}[Theorem 3.1.1 in \cite{vershynin_2018}]\label{Concentration of norm of subgaussian vector}
Let $\bX=(X_1,\dots,X_d)\in\mathbb{R}^d$ be a random vector with independent $X_i$. Suppose that $\{X_j\}_{j\in[d]}$ are mean-zero subgaussian random variables with variance $1$ and the same subgaussian parameters $\beta,\kappa$. Then there exists a constant $C>0$ which depends on subgaussian parameters such that
\begin{align*}
    \mathbb{P}\left( |\|\bX\|_2 - \sqrt{d}|\geq t \right)\leq 2\exp (-Ct^2).
\end{align*}
\end{lemma}

\end{document}